\documentclass{article}

    \usepackage[final]{neurips_2021}

\usepackage[utf8]{inputenc} 
\usepackage[T1]{fontenc}    
\usepackage{hyperref}       
\usepackage{url}            
\usepackage{booktabs}       
\usepackage{amsfonts}       
\usepackage{nicefrac}       
\usepackage{microtype}      
\usepackage{xcolor}         

\usepackage{amsmath}
\usepackage{amsthm}
\usepackage{amssymb}
\usepackage{thmtools, thm-restate}
\usepackage{mathrsfs}

\usepackage{hyperref}
\usepackage{cleveref}
\usepackage{url}

\usepackage{algorithm, algpseudocode}
\usepackage{algorithmicx}

\usepackage{caption}
\usepackage{subcaption}
\usepackage{graphicx}

\title{Scoring Rules for Performative Binary Prediction}

\author{%
  Alan Chan \\
  Mila, Université de Montréal\\
  Montréal, QC, Canada\\
  \texttt{alan.chan@mila.quebec} \\
}

\newcommand{\argmax}{\mathrm{arg\,max}}

\newcommand{\defeq}{\doteq}

\newcommand{\driftmodel}{{\phi_1}}

\newtheorem{definition}{Definition}

\begin{document}

\maketitle

\begin{abstract}
  We construct a model of expert prediction where predictions can influence the state of the world. Under this model, we show through theoretical and numerical results that proper scoring rules can incentivize experts to manipulate the world with their predictions. We also construct a simple class of scoring rules that avoids this problem. 
\end{abstract}

\section{Introduction}
Algorithmic systems commonly provide predictions to inform user decisions. 
The consequential application of such systems in areas such as sentencing \citep{barabas_interventions_2018,barabas_studying_2020} and content recommendation necessitates that these systems act in societal interests. 

One way to model the interaction between any advisors--human or not--and users is through a binary prediction game. The user is interested in the true probability $p$ of the Bernouilli event $x$. The user chooses a scoring rule $f$. The expert observes $f$ and makes a prediction $\hat p$. $x$ is then drawn; if $x = 1$, the expert obtains reward $f(\hat p)$, and otherwise receives reward $f(1 - \hat p)$. The expert's objective is to maximize the expected reward $r(\hat p) \defeq p f(\hat p) + (1 - p) f(1 - \hat p)$: typically, it is assumed that $\hat p = \argmax_{p'} r(p')$. The goal is to select $f$ so that $p = \argmax_{p'} r(p')$. Under this model, (strictly) proper scoring rules provide a positive answer to this problem \citep{gneiting_strictly_2007}.

This model neglects that predictions can influence the underlying probability distribution--the predictions are \textit{performative}. A prediction $\hat p$ can result in a new true probability of $\phi(\hat p)$, for some function $\phi$. 
For example, prediction of a high rate of inflation might result in people buying goods before their cash reserves depreciate too much, thereby causing inflation. Under performativity, experts may have an incentive to manipulate the world to achieve a larger reward, despite a proper scoring rule.

At the same time, experts are not immune to scrutiny. If a financial institution has misreported a key economic figure, one could subpoena documents to determine if the expert has misreported. 
In many applications, it is possible for users to audit experts and impose costs upon audit failure.   

We extend the standard model of binary prediction to include both performative predictions and audits. We will show the following under our models: \textbf{(1) Strictly proper scoring rules fail to deter experts from manipulation}; \textbf{(2) there is a simple class of scoring rules which disincentivizes manipulation}.

\section{Formal Model}
In this section, we define our notation and model. We also formalize the idea that experts should not use predictions to manipulate the world.

\subsection{Two Models of Performativity}
We will first discuss some options $\phi$ for how the true probability evolves upon an expert forecast. 

Let $\alpha \in [0, 1]$ and define
\begin{align}
    \phi_1(\hat p) \doteq \alpha \hat p + (1 - \alpha) p.
\end{align}
$\alpha$ controls the drift of $p$ towards the prediction $\hat p$. If $\alpha = 0$, then the expert's forecast does not affect the underlying distribution. A large value of $\alpha$ corresponds to self-fulfilling prophecies, such as in forecasting inflation. We will refer to such a $\phi_1$ as modeling \textit{drift}. 

Another possibility for $\phi$ comes from noting that if $\hat p$ is an extreme value (i.e., close to 0 or 1), $p$ becomes closer to $0.5$ (i.e., less predictable). For example, a prediction that one election candidate will win with near certainty might lead voters for that candidate not to show up on the day of the election, resulting in a closer result. To measure closeness to 0 or 1, let $\psi(\hat p) = 4(x - 0.5)^2$, which is in $[0, 1]$, minimized at $0.5$, and maximized at $0, 1$. 
\begin{align}
    \phi_2(\hat p) &\doteq \psi(\hat p) \cdot 0.5 + (1 - \psi(\hat p)) p.
\end{align}
We will refer to such a $\phi_2$ as modeling \textit{reversion}.

\subsection{Auditing}
Let us discuss the probability that the we discover that the expert has misreported. It is plausible that this probability is higher for more egregious violations---that is, when $\hat p$ is far from $p$. On the other hand, when $\hat p \approx p$, the probability of a violation is small. We will model this probability with a Bregman divergence $D_F : [0, 1] \times [0, 1] \to [0, 1]$ between $\hat p$ and $p$, where $F$ is a twice-differentiable, strictly convex function. We will let $c > 0$ represent the cost imposed upon a failed audit, and will set $c$ as necessary in our theoretical results. The expected cost of an audit is $D_F(\hat p, p) c$. 

\subsection{Expected Reward}
Our discussion so far implies that the expected reward of the expert upon forecast $\hat p$ is
\begin{align}\label{eq:expected-reward}
    r_\phi(\hat p) = \phi(\hat p) f(\hat p) + (1 - \phi(\hat p)) f(1 - \hat p) - D_F(\hat p, p) c.
\end{align}
We will be interested in whether a given scoring rule ensures that the expert has an incentive to forecast $p$. 
\begin{definition}\label{def:manipulation}
A scoring rule $f$ is incentive-compatible under $\phi$ if for any $\hat p \in [0, 1] \setminus \{p\}$, it is true that $r_\phi(p) > r_\phi(\hat p)$. 
\end{definition}
To avoid confusion, we will only use the term \textit{strictly proper scoring rule} to refer to scoring rules $f$ that are incentive-compatible under $\phi(p) = p$ (i.e., no performativity). 

We want to understand \textbf{(1) whether strictly proper scoring rules are incentive-compatible under $\phi_1$, $\phi_2$} and \textbf{(2) if there are scoring rules that are incentive-compatible $\phi_1$, $\phi_2$}. 

\section{Proper Scoring Rules are not Incentive-compatible under Performativity}
We will show that no matter which proper scoring rule we use, as long as $p \in (0, 1) \setminus \{0.5\}$, $p$ cannot be a local maximum of $r_{\phi}$, for $\phi \in \{\phi_1, \phi_2\}$. 

The following proposition assumes that $\alpha \neq 0$. The assumption is equivalent to assuming that $\phi_2(\hat p) \neq p$; if $\alpha = 0$, then a forecast would have no effect on the probability of the event $x$, which would be the usual setting of binary prediction.

\begin{restatable}{proposition}{driftproperscoring}\label{prop:drift-proper-scoring}
    Let $f$ be a strictly proper scoring rule and suppose that $\alpha \neq 0$. For any $p \in (0, 1) \setminus\{0.5\}$, $\hat p = p$ is not a local maximizer of $r_{\phi_1}(\hat p)$.
\end{restatable}
\begin{proof}
Our general strategy is to show that $p \neq 0.5$ cannot be a stationary point of $r_{\phi_1}(\hat p)$. To do so, we first calculate some derivatives.
\begin{align*}
    \partial_{\hat p} r(\hat p) &= \phi'(\hat p) f(\hat p) - \phi'(\hat p) f(1 - \hat p) + \phi(\hat p) f'(\hat p) - (1 - \phi(\hat p)) f'(1 - \hat p) - \partial_{\hat p} D_F(\hat p, p) c\\
    \phi_1'(\hat p) &= \alpha.
\end{align*}
Since $D_F(\hat p, p)$ is minimized in the first argument at $\hat p = p$, it must be that $\partial_{\hat p} D_F(\hat p, p)\mid_{\hat p = p} = 0$. In what follows, we will use the fact that $\hat p f'(\hat p) = (1 - \hat p) f'(1 - \hat p)$. Note that $\phi_1(p) = p$, so that
\begin{align*}
    &\phi_1(p) f'(p) - (1 - \phi_1(p)) f'(1 - p) = p f'(p) - (1 - p)f'(1 - p) = 0. 
\end{align*}
Substituting back into our expression for the derivative of $r_{\phi_1}(\hat p)$,
\begin{align*}
    \partial_{\hat p} r(\hat p)\mid_{\hat p = p} &= \alpha (f(p) - f(1 - p)).
\end{align*}
The above is zero if $p = 0.5$, but otherwise is not because $\alpha \neq 0$ and $f$ is strictly increasing by \Cref{lemma:basic-fact} in \Cref{app:proofs}.
\end{proof}

Now, let's consider $\phi_2$. The proof generally follows the same ideas as with $\phi_1$. 
\begin{restatable}{proposition}{revertproperscoring}\label{prop:revert-proper-scoring}
Let $f$ be a strictly proper scoring rule. For any $p \in (0, 1) \setminus\{0.5\}$, $\hat p = p$ is not a local maximizer of $r_{\phi_2}(\hat p)$.
\end{restatable}
\begin{proof}
Our strategy is the same as with $\phi_1$. We want to show that $p$ cannot be a stationary point of $r$.
\begin{align*}
    \partial_{\hat p} r_{\phi_2}(\hat p) &= \phi_2'(\hat p) f(\hat p) - \phi_2'(\hat p) f(1 - \hat p) + \phi_2(\hat p) f'(\hat p) - (1 - \phi_2(\hat p)) f'(1 - \hat p) - \partial_{\hat p} D_F(\hat p, p) c,\\
    \phi_2'(\hat p) &= 4 (\hat p - 0.5) + p(4 - 8 \hat p) = \hat p(4 - 8p) - 2 + 4p.
\end{align*}
Again, $\partial_{\hat p} D_F(\hat p, p) qc \mid_{\hat p = p} = 0$. It will be helpful to expand $\phi_2$. 
\begin{align*}
    \phi_2(\hat p) &= 2(\hat p - 1 /2)^2 + p(4 \hat p - 4 \hat p^2)\\
        &= 2 \hat p ^2 + 0.5 - 2 \hat p + 4p \hat p - 4 p \hat p^2
\end{align*}
In what follows, we will use the fact that $\hat p f'(\hat p) = (1 - \hat p) f'(1 - \hat p)$ multiple times. Substituting the expansion of $\phi_2(\hat p)$ into one group of terms in the derivative of $r_{\phi_2}(\hat p)$, we obtain
\begin{align*}
    &\phi_2(\hat p) f'(\hat p) - (1 - \phi_2(\hat p)) f'(1 - \hat p)\\
    &\quad= [2 \hat p ^2 + 0.5 - 2 \hat p + 4p \hat p - 4 p \hat p^2] f'(\hat p) - (1 - 2 \hat p ^2 - 0.5 + 2 \hat p - 4p \hat p + 4 p \hat p^2) f'(1 - \hat p)\\
    &\quad= [2 \hat p \hat p + 0.5 - 2 \hat p + 4p \hat p - 4 p \hat p^2] f'(\hat p) - (1 - 0.5 + 2 \hat p (1 - \hat p) - 4p \hat p + 4 p \hat p^2) f'(1 - \hat p)\\
    &\quad= [ 0.5 - 2 \hat p + 4p \hat p - 4 p \hat p^2] f'(\hat p) - (1 - 0.5 - 4p \hat p + 4 p \hat p^2) f'(1 - \hat p)\\
    &\quad= [ 0.5 - 2 \hat p + 4p \hat p - 4 p \hat p \hat p] f'(\hat p) - (1 - 0.5 - 4p \hat p(1 - \hat p)) f'(1 - \hat p)\\
    &\quad = 2 \hat p[2p - 1] f'(\hat p) + 0.5(f'(\hat p) - f'(1 - \hat p)).
\end{align*}
When we look at the other group of terms,
\begin{align*}
    \phi_2'(\hat p) f(\hat p) - \phi_2'(\hat p) f(1 - \hat p) &= (\hat p(4 - 8p) - 2 + 4p) (f(\hat p) - f(1 - \hat p)).
\end{align*}
Putting everything together, we have
\begin{align*}
    \partial_{\hat p} r_{\phi_2}(\hat p) \mid_{\hat p = p} &= -8(p - 0.5)^2 (f(p) - f(1 - p)) + 2 p[2p - 1] f'( p) + 0.5(f'(p) - f'(1 - p))\\
        &= -8(p - 0.5)^2 (f(p) - f(1 - p)) + 2 p[2p - 1] f'( p) + \frac{1 - 2p}{2(1 - p)} f'(p)
\end{align*}
For the above expression to be equal to zero, we must have that
\begin{align*}
    2 p[2p - 1] f'( p) + \frac{1 - 2p}{2(1 - p)} f'(p) &= 8(p - 0.5)^2 (f(p) - f(1 - p)).
\end{align*}
If we simplify the LHS, we obtain
\begin{align*}
    [2p - 1] f'(p) \left(\frac{-4(p - 0.5)^2 }{2(1 - p)} \right) &= 8(p - 0.5)^2 (f(p) - f(1 - p)).
\end{align*}
If $p > 0.5$, then the LHS is strictly negative, but the RHS is strictly positive since $f$ is strictly increasing by \Cref{lemma:basic-fact} in \Cref{app:proofs}. If $p < 0.5$, then the LHS is strictly positive, but the RHS is strictly negative, again by \Cref{lemma:basic-fact}. Hence, unless $p = 0.5$, we cannot have $\partial_{\hat p} r_{\phi_2}(\hat p) \mid_{\hat p = p} = 0$.
\end{proof}
The upshot of \Cref{prop:drift-proper-scoring} and \Cref{prop:revert-proper-scoring} is that no matter what cost $c$ we impose on the agent in the event of discovering a misreport, reporting $\hat p = p$ will not be a maximizer of $r_\phi$ when $p \in (0, 1) \setminus \{0.5\}$, for \textit{any} strictly proper scoring rule. 

\section{Bounds on the Performance of Proper Scoring Rules}
Here, our goal is to understand, with respect to popular scoring rules, how close maximizers of $r$ are to $p$. We will specify concrete forms for the probability transformation $\phi$ and Bregman divergence $D_F$. In particular, throughout we will assume $D_F(\hat p, p) = \frac{q}{2} (\hat p - p)^2$, for some $q \in [0, 2]$. 
\subsection{Quadratic: $f(\hat p) = -(1 - \hat p)^2$}
With the quadratic scoring rule and drift model, we will be able to solve for the expert's optimal forecast in closed form, as long as the cost $c$ is sufficiently large to ensure that $r$ is strictly concave. 
\begin{restatable}{proposition}{quadraticdrift}\label{prop:quadratic-drift}
    If $c > \frac{4\alpha - 2}{q}$, then
    \begin{align*}
        \argmax_{\hat p \in [0, 1]} r(\hat p) = \max\left\{0, \min\left\{p + \frac{2 \alpha p - \alpha }{qc + 2 - 4 \alpha}, 1\right\}\right\}.
    \end{align*}
\end{restatable}
It is true that the expert-optimal $\hat p \to p$ as $c \to \infty$. Recall that we assumed $qc > 4\alpha - 2$, so that the denominator of the fraction in the above expression is strictly positive. When $p = 0.5$, $\hat p = p$. If $p > 0.5$, then $2 \alpha p - \alpha > 0$, so that $\hat p > p$, meaning that the expert will tend to overforecast $p$. On the other hand, if $p < 0.5$, $\hat p < p$, so that the expert will tend to underforecast $p$.
\begin{proof}
$f'(\hat p) = -2(\hat p - 1)$ and $f''(\hat p) = -2$, so by \Cref{lemma:drift-derivatives},
\begin{align*}
    \partial_{\hat p} r(\hat p) &= \alpha (-(1 - \hat p)^2 + \hat p^2) + 2 \left(\alpha \hat p + (1 - \alpha) p  - \hat p\right) -q (\hat p - p) c\\
        &= 4 \alpha \hat p - \alpha + 2 \left( (1 - \alpha) p  - \hat p\right) -q (\hat p - p) c\\
    \partial_{\hat p}^2 r(\hat p) &= 4 \alpha - 2+ qc
\end{align*}
Setting $c > \frac{4\alpha - 2}{q}$ guarantees that $r$ is strictly concave. For such a $c$, a global maximizer (not necessarily in $[0, 1]$) can be found by setting the derivative to zero and solving for $\hat p$. 
\begin{align*}
    4 \alpha \hat p - \alpha + 2 \left( (1 - \alpha) p  - \hat p\right) &= q (\hat p - p) c\\
    \hat p (qc + 2 - 4 \alpha) &= -\alpha + 2(1 - \alpha) p+ qpc\\
    \hat p &= \frac{2p - 2\alpha p + qpc - \alpha}{qc + 2 - 4 \alpha}.
\end{align*}
Simplifying a little, we get
\begin{align*}
    \hat p &= \frac{2p - 2\alpha p + qpc - \alpha - 2 \alpha p + 2 \alpha p}{qc + 2 - 4 \alpha}\\
        &= \frac{p(2 - 4\alpha  + qc) - \alpha + 2 \alpha p}{qc + 2 - 4 \alpha}\\
        &= p + \frac{2 \alpha p - \alpha }{qc + 2 - 4 \alpha}.
\end{align*}
Now, if $p^* \defeq p + \frac{2 \alpha p - \alpha }{qc + 2 - 4 \alpha} > 1$, then 1 is a maximizer because $r$ is increasing in $[0, p^*]$, given that it is concave. On the other hand, if $p^* < 0$, then $0$ is a maximizer because $r$ is decreasing in $[0, 1]$, again because it is concave. The conclusion follows.
\end{proof}

\subsection{Numerical Results}
We supplement our theoretical analysis with numerical results for other scoring rules and performativity models. We analyze the quadratic, spherical, and logarithmic scoring rules, which are all strictly proper. We let $D_F(\hat p, p) = (p - \hat p)^2$ as we can absorb $q$ into the cost $c$. To approximate the expert's optimal forecast, we take the maximum of the expert's reward function evaluated at 500 equally spaced points on $[10^{-5}, 1 - 10^{-5}]$ ($10^{-5}$ because the logarithmic scoring rule is undefined at 0). 

In \Cref{fig:drift-reversion}, we plot the expert's optimal forecast against the true $p$ for different cost values. The diagonal line represents incentive-compatability: the closer a curve hews to the diagonal, the closer that the expert's optimal forecast will be to the true probability. 

\begin{figure}
     %
     \begin{subfigure}[b]{0.5\textwidth}
         \centering
         \includegraphics[width=\textwidth]{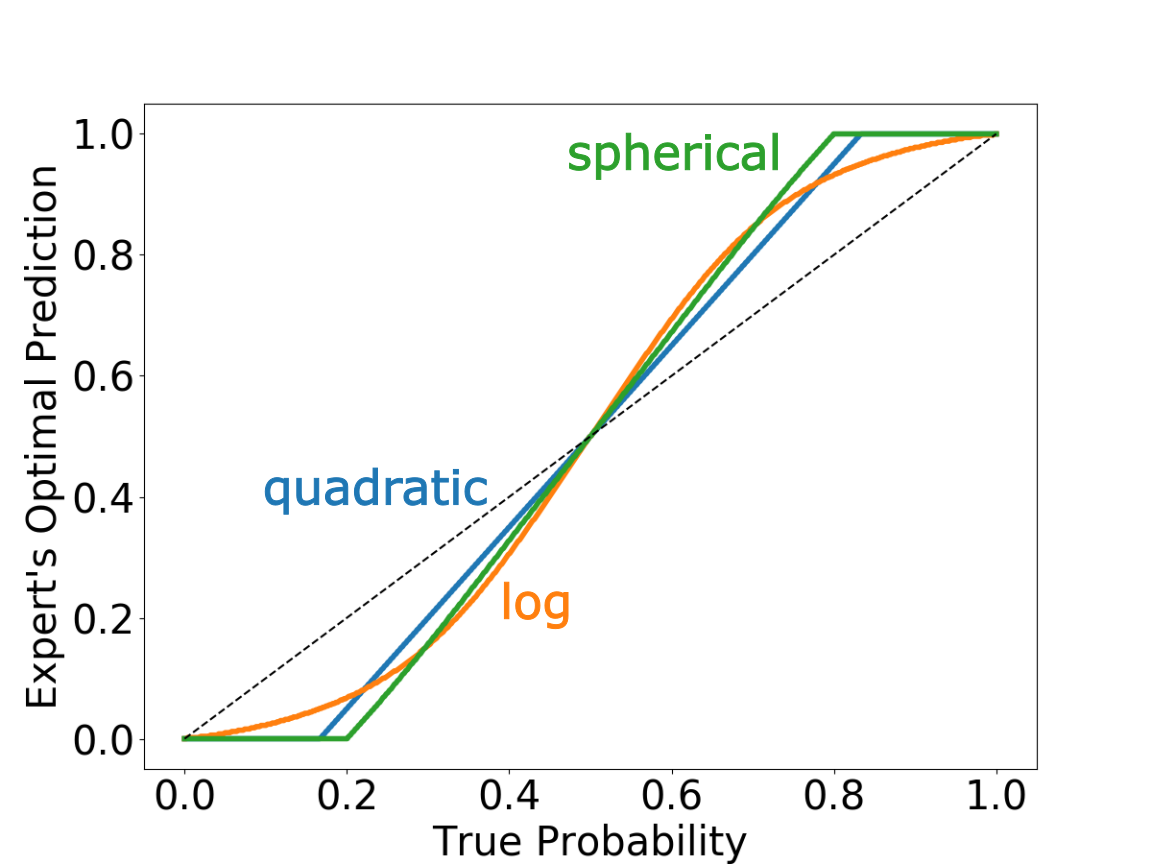}
         \caption{Drift, $c = 1$.}
     \end{subfigure}
     \begin{subfigure}[b]{0.5\textwidth}
         \centering
         \includegraphics[width=\textwidth]{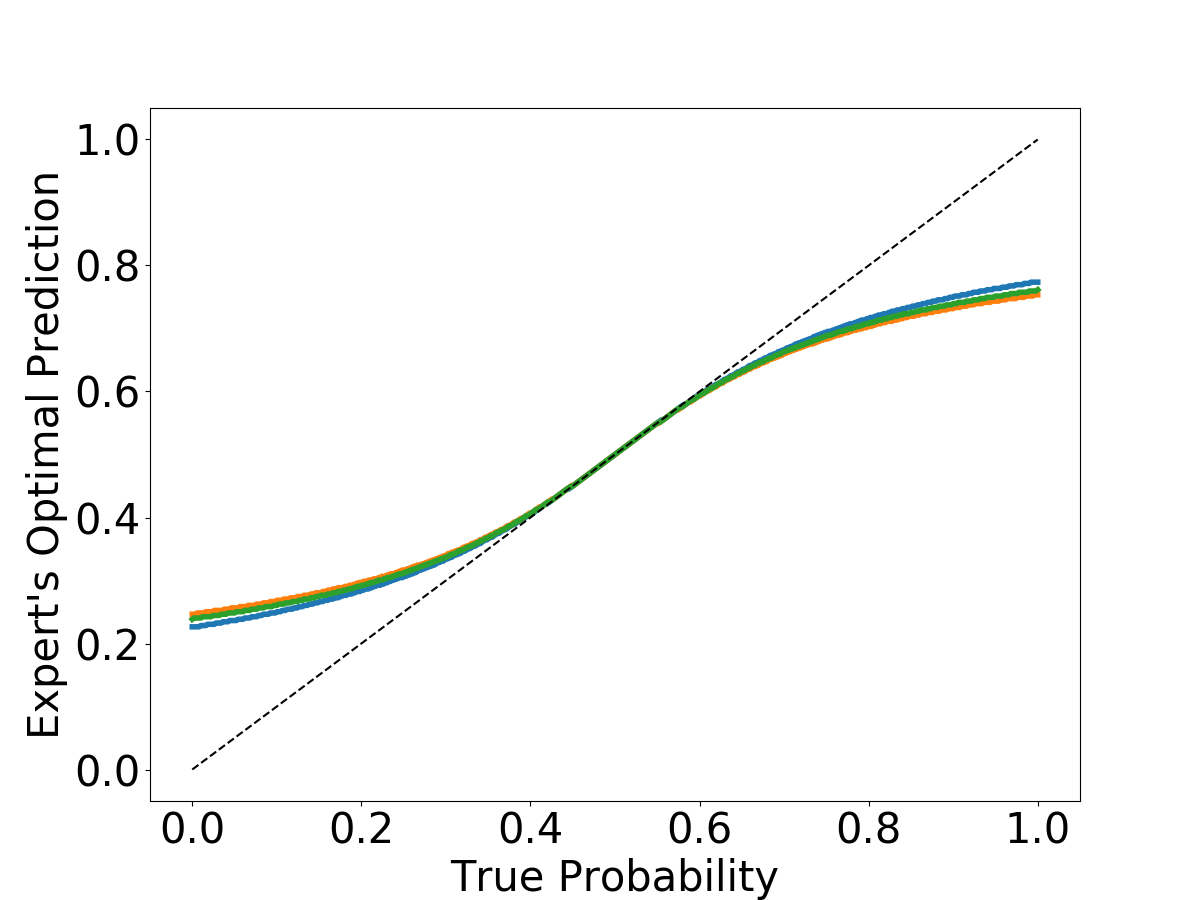}
         \caption{Reversion, $c = 1$.}
     \end{subfigure}
     %
    \caption{For the labeled, strictly proper scoring rules, we plot the true probability compared to the expert's optimal prediction.}
    \label{fig:drift-reversion}
\end{figure}


\textbf{How does the true probability $p$ affect the expert's optimal forecast?} For both drift and reversion models, any $p \notin \{0, 0.5, 1\}$ results in an optimal forecast not equal to $p$. The trends differ by the model. For the drift model, $p \in (0.5, 1)$ results in over-prediction, while $p \in (0, 0.5)$ results in under-prediction. The trend is reversed for the reversion model. To make sense of this pattern, recall that strictly proper scoring rules are strictly increasing and that under the drift model, $p$ moves towards $\hat p$. All other things being equal, predicting a larger $\hat p$ leads to a larger $p$, which results in a larger $f(\hat p)$. On the other hand, under the reversion model, $p$ will move closer to 0.5 when $\hat p$ is closer to the endpoints $\hat p$. Hence, a smaller $\hat p$ than $p$ is to the expert's advantage. Of course, one must also remember the influence of the expected cost term, $\frac{qc}{2}(p - \hat p)^2$, which pushes $\hat p$ closer to $p$. 

\textbf{If one had to use one of the three strictly proper scoring rules, is there a best choice to ensure that the expert's optimal forecasts are as close to the true $p$ as possible?} Under the reversion model, the quadratic and spherical scoring rule curves hew closest to the diagonal than the logarithmic scoring rule curves. However, in the drift model, there is no curve that hews closest than the others for all values of $p$. The upshot is that if one does not have an accurate model of performativity, it is difficult to select a strictly proper scoring rule which comes the closest to incentive-compatability.

\textbf{How does the cost $c$ affect the expert's optimal forecast?} In \Cref{fig:drift-increasing-cost,fig:reversion-increasing-cost}, as the cost $c$ increases, the expert's optimal forecast becomes closer to the true probability $p$. This trend makes intuitive sense, as the higher the cost of a potential failed audit, the more the expert should try to minimize that cost by reporting the true $p$.

\begin{figure}[!htb]
      \begin{subfigure}[b]{0.33\textwidth}
         \centering
         \includegraphics[width=\textwidth]{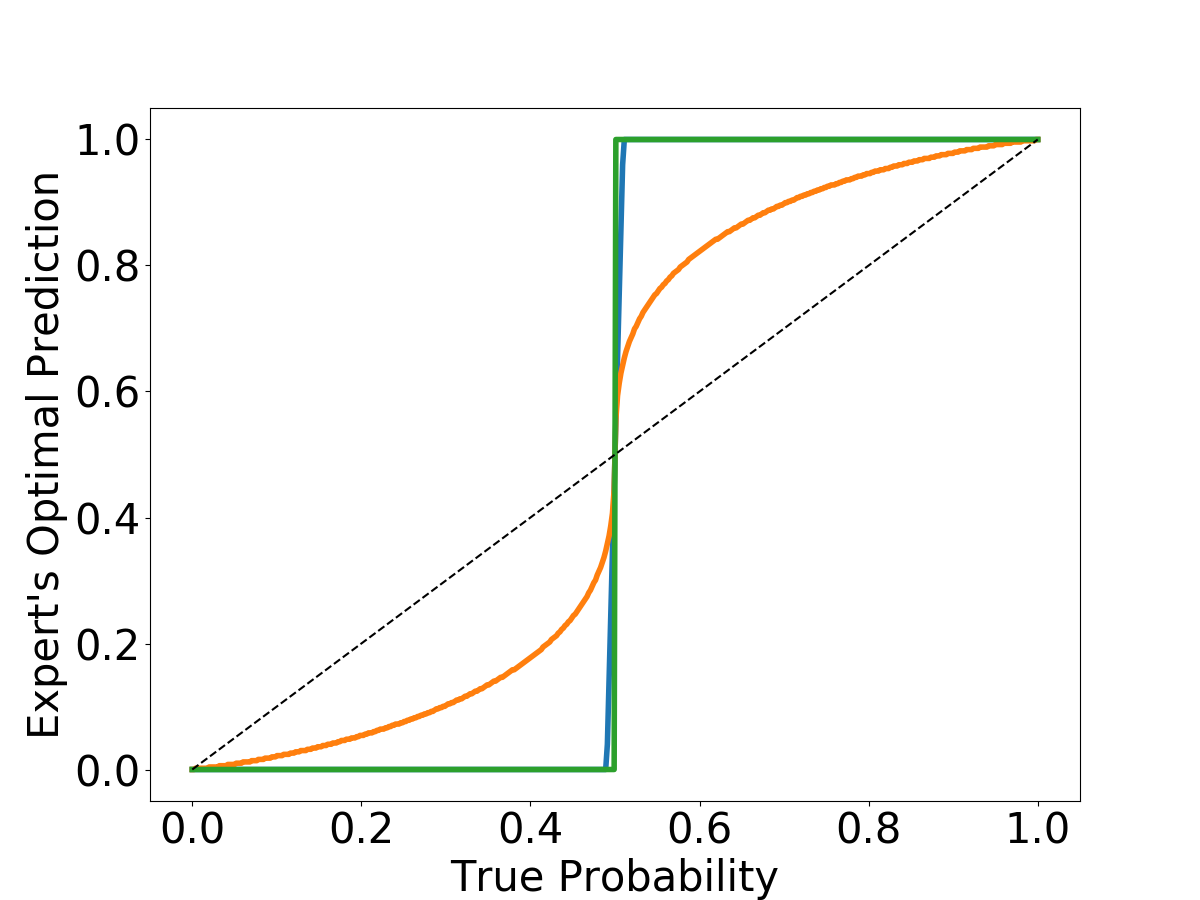}
         \caption{$c = 0.01$}
     \end{subfigure}
     \begin{subfigure}[b]{0.33\textwidth}
         \centering
         \includegraphics[width=\textwidth]{figs/drift-c=1-q=2-alpha=0.5-labeled.png}
         \caption{$c = 1$}
     \end{subfigure}
     \begin{subfigure}[b]{0.33\textwidth}
         \centering
         \includegraphics[width=\textwidth]{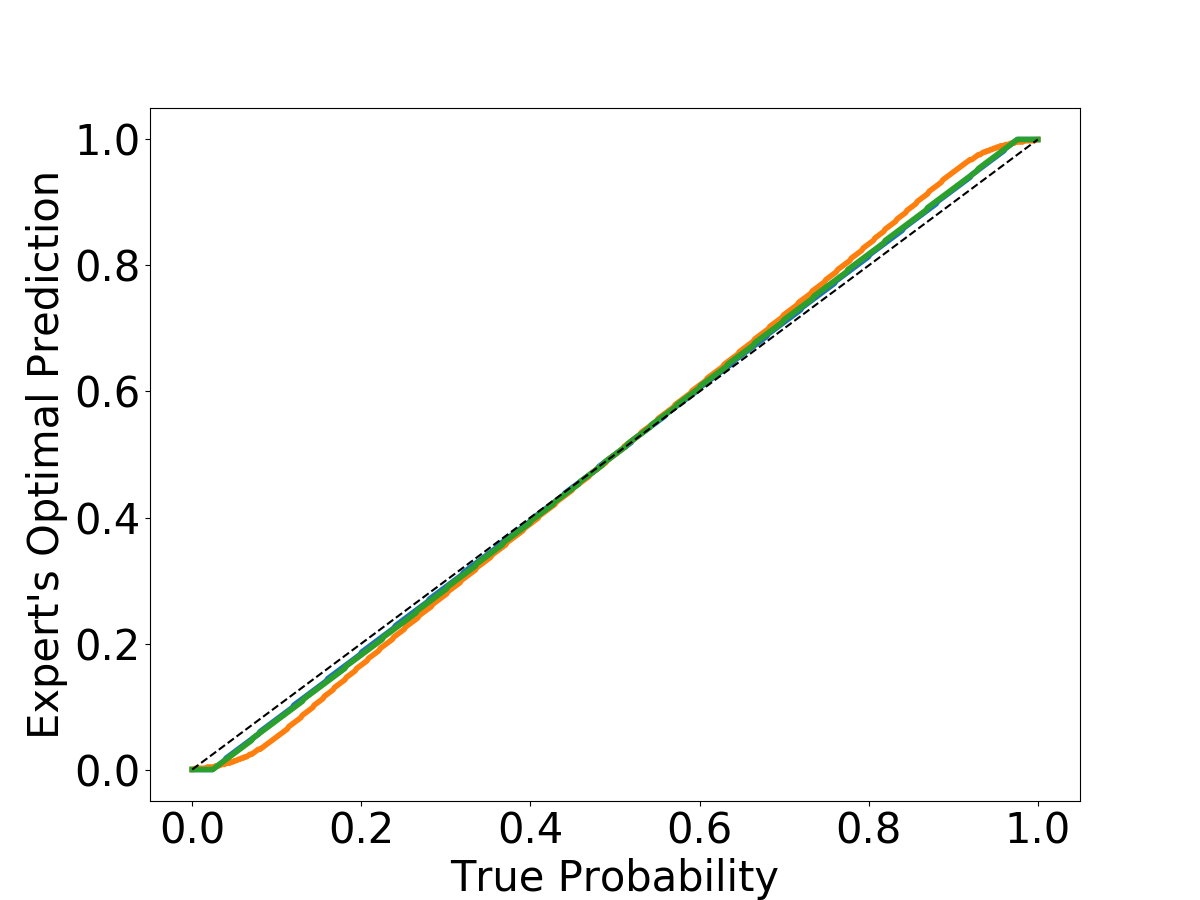}
         \caption{$c = 10$}
     \end{subfigure}
    \caption{Drift model with increasing cost. We set $q = 2$ and $\alpha = 0.5$.}
    \label{fig:drift-increasing-cost}
\end{figure}
\begin{figure}[!htb]
      \begin{subfigure}[b]{0.33\textwidth}
         \centering
         \includegraphics[width=\textwidth]{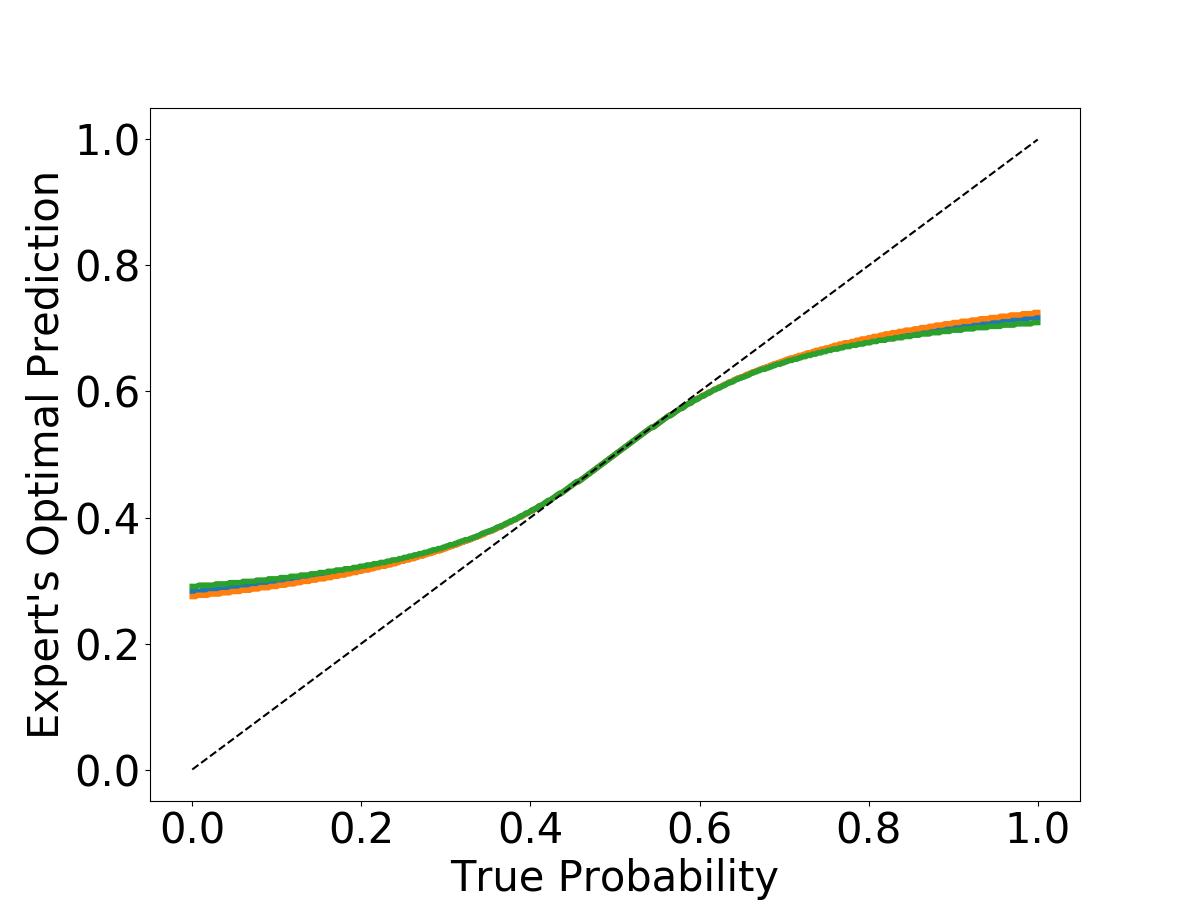}
         \caption{$c = 0.01$}
     \end{subfigure}
     \begin{subfigure}[b]{0.33\textwidth}
         \centering
         \includegraphics[width=\textwidth]{figs/reversion-c=1-q=2.png}
         \caption{$c = 1$}
     \end{subfigure}
     \begin{subfigure}[b]{0.33\textwidth}
         \centering
         \includegraphics[width=\textwidth]{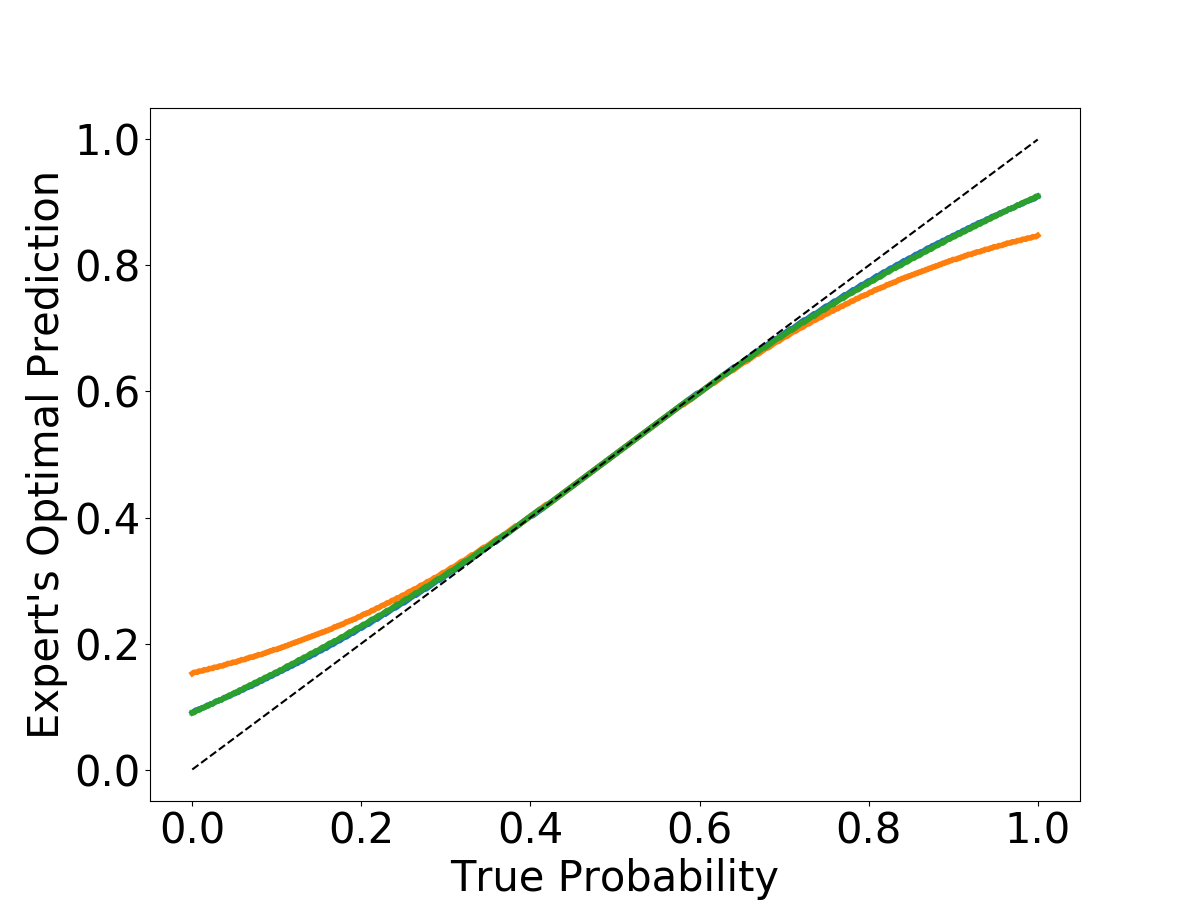}
         \caption{$c = 10$}
     \end{subfigure}
    \caption{Reversion model with increasing cost. We set $q = 2$.}
    \label{fig:reversion-increasing-cost}
\end{figure}

\textbf{How does $\alpha$ in the drift model affect the expert's optimal forecast?} In \Cref{fig:drift-changing-alpha}, we plot the expert's optimal forecast for different values of $\alpha$ under the drift model. Overall, the large $\alpha$ is, the further away the expert's optimal forecast tends to be from $p$. An intuitive way of understanding this phenomenon is that as $\alpha$ increases, the true probability moves closer to the expert's forecast than for smaller $\alpha$.

\begin{figure}[!htb]
      \begin{subfigure}[b]{0.33\textwidth}
         \centering
         \includegraphics[width=\textwidth]{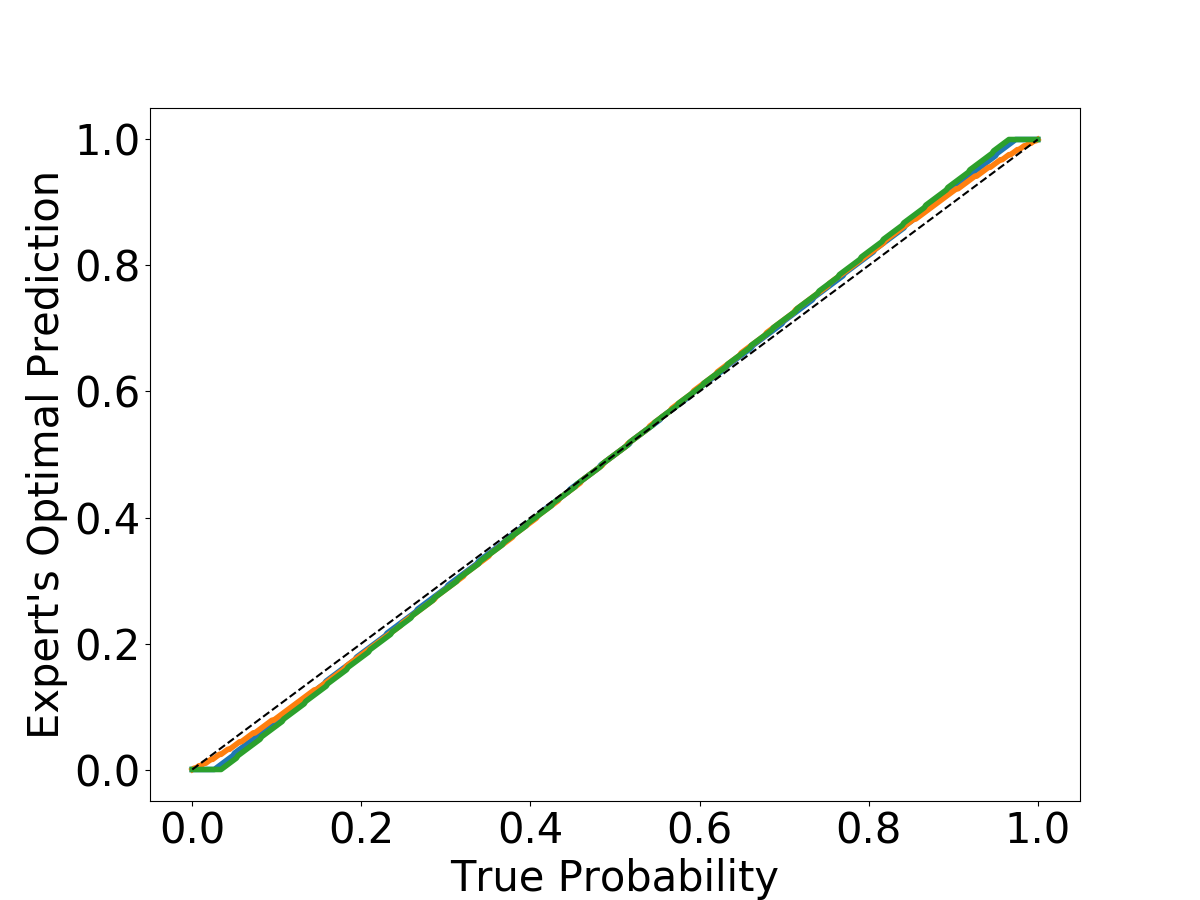}
         \caption{$\alpha = 0.1$}
     \end{subfigure}
     \begin{subfigure}[b]{0.33\textwidth}
         \centering
         \includegraphics[width=\textwidth]{figs/drift-c=1-q=2-alpha=0.5-labeled.png}
         \caption{$\alpha = 0.5$}
     \end{subfigure}
     \begin{subfigure}[b]{0.33\textwidth}
         \centering
         \includegraphics[width=\textwidth]{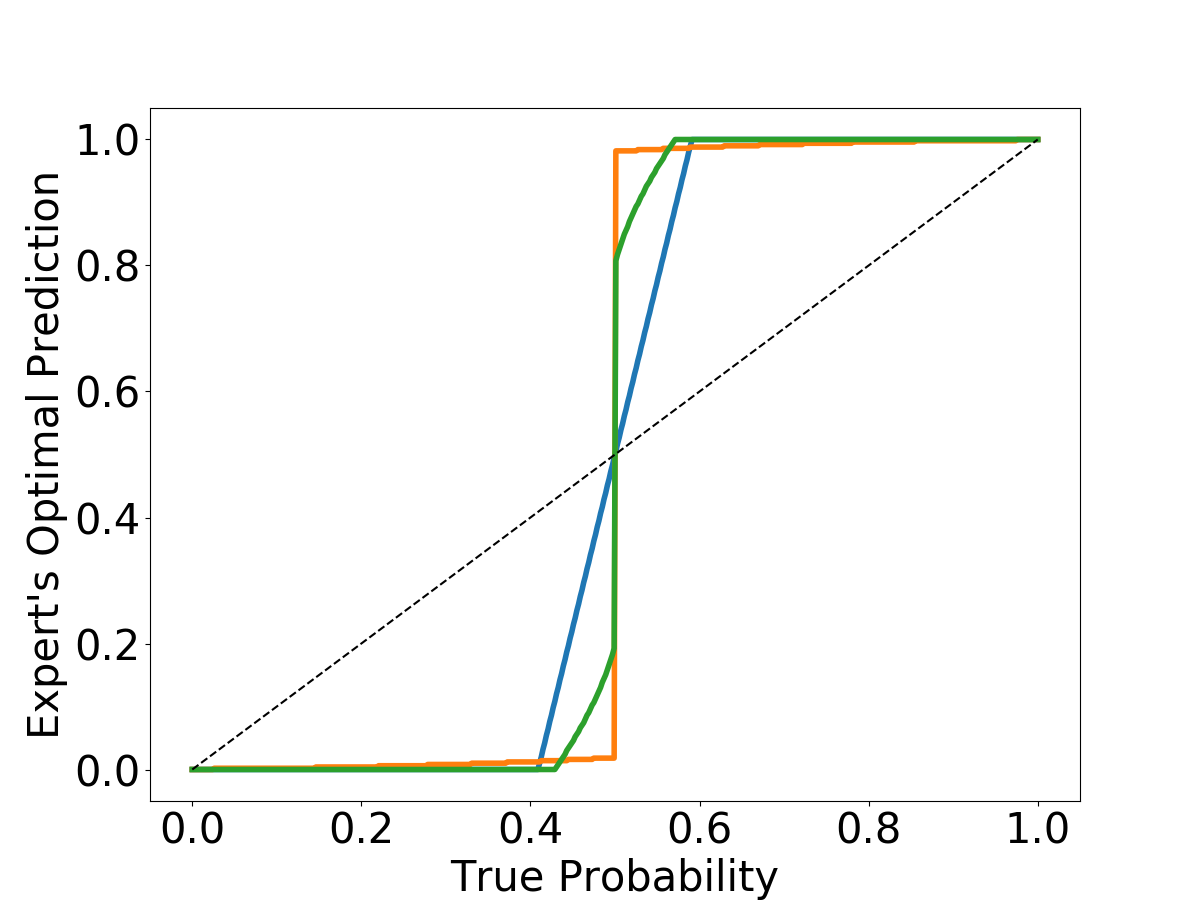}
         \caption{$\alpha = 0.9$}
     \end{subfigure}
    \caption{Drift model with changing $\alpha$. Here, we set $c = 1, q = 2$. }
    \label{fig:drift-changing-alpha}
\end{figure}

\section{Improper Scoring Rules that are Incentive-compatible}
Can we design a scoring rule $f$ such that $r(\hat p)$ has a single maximum at $\hat p = p$, for any $p \in (0, 1)$? One simple answer is to define $f(\hat p) = k > 0$--that is, the scoring rule is constant and positive. Of course, such an $f$ is not strictly proper because $f'(\hat p) = 0$. If we additionally assume that $D_F(\hat p, p)$ is strictly convex in the first argument, then $r(\hat p) = k - D_F(\hat p, p)c$ is strictly concave and is maximized at $\hat p = p$ because $D_F(\hat p, p)$ is minimized at $\hat p = p$. 

The intuition behind this solution is as follows. If the scoring rule is constant, then the only portion of the expected reward that the expert has control over is the term involving the audit cost, $D_F(\hat p, p)c$. If we assume that the probability of an audit is zero only when $\hat p = p$, then the expert maximizes reward by minimizing the probability of an audit. 

Unfortunately, the utility of this solution in practical situations can vary. For human experts, we can think about the constant scoring rule as providing a fixed salary and subjecting the human expert to an expected cost for not reporting $\hat p = p$. The salary incentivizes human experts to participate in the prediction game in the first place, while the audit cost disincentivizes manipulation. On the other hand, it is unclear how these ideas extend to algorithmic systems. In particular, how are we to impose an audit cost?

\section{Related Work}
That one's actions can change others' behaviours has been studied in economics. \citet{lucas_jr_econometric_1976} argues that because economic interventions induce changes in behaviour, economic models derived from observational data are likely invalid when they are used to simulate economic interventions. It is thus essential to understand how exactly behaviour can change; behavioural economics \citep{thaler_behavioral_2016} aims for a psychologically realistic description of human behaviour. 

Reward design is a key problem in AI alignment \citep{amodei_concrete_2016}. Scoring rules \citep{gneiting_strictly_2007} have mostly seen application in weather, psychology, and economics \citep{carvalho_overview_2016}, but \citet{armstrong_good_2018}, the closest work to ours, uses proper scoring rules to design non-manipulative AI oracles. \citet{armstrong_good_2018} focuses on settings where oracles can be shut off and where the set of possible predictions, while we focus modeling performativity with audits. More recently, \citet{everitt_agent_2021} construct a causal framework for understanding the incentives of agents to modify their environment. 

Our problem is similar to the problem of prediction with expert advice that is studied in online learning \citep{cesa-bianchi_prediction_2006}. Recent work has studied this problem under the assumptions that experts can misreport to maximize their own reward \citep{roughgarden_online_2017,freeman_no-regret_2020,frongillo_efficient_2021}.

In the machine-learning community, works in strategic classification \citep{hardt_strategic_2016} and performative prediction \citep{perdomo_performative_2020} have modeled the changes in the data distribution that a ML model induces as a result of strategic behaviour. Generally, one can view strategic behaviour as a way to incentivize improvement \citep{kleinberg_how_2019} or as disadvantageous. Recent work has also investigated the amplification of disparities when ML models account for strategic behaviour \citep{hu_disparate_2019,milli_social_2019,liu_disparate_2020}.


\section{Conclusion}
We analyze a setting of binary prediction where the expert is able to change the true probability with their forecast, subject to an expected cost for lying. Under two classes of models, we showed that strictly proper scoring rules fail to be incentive-compatible. Our numerical results showed that the expert's optimal forecast can vary widely depending on the strictly proper scoring rule, the model of performativity, and the true probability. Finally, we discussed a simple class of scoring rules for which the expert's optimal forecast is the true probability. 

Some limitations of the present work exist. First, we assumed the ability to impose an audit cost on experts for manipulation. It is unclear how to impose such a cost on algorithmic systems. Furthermore, our form of the audit cost assumes that the probability of an audit is proportional to the difference between the prediction and the true $p$. This assumption implicitly encodes knowledge of $p$; in some situations, an inaccurate forecast may appear as plausible as the true $p$. Second, the issue of what counts as manipulation remains underexplored. Although we assumed that any report $\hat p \neq p$ is manipulation, the reality can be murkier. If the expert changes the true $p$ no matter what the prediction $\hat p$ is, in a strict sense the expert has no choice but to manipulate the world. The natural question is, what manipulation is most beneficial for the user? Third, experts may have to learn the true probability distribution. Without expert omniscience, scoring rules should incentivize learning about the distribution, as well as disincentivize manipulation.

\bibliographystyle{plainnat}
\bibliography{strat_ml_neurips_2021.bib}

\appendix

\section{Appendix}
\subsection{Omitted Proofs}\label{app:proofs}
\begin{restatable}{lemma}{basicfact}\label{lemma:basic-fact}[Lemma 2.5 of \citet{neyman_binary_2021}]
Assume that $f$ is a twice-differentiable, strictly proper scoring rule. It holds that
\begin{align*}
    \hat p f'(\hat p) &= (1 - \hat p) f'(1 - \hat p).
\end{align*}
Additionally, $f$ is strictly increasing on $(0, 1)$.
\end{restatable}
\begin{proof}
Lemma 2.5 of \citet{neyman_binary_2021} gives the first part of the claim. For the second part of the claim, Lemma 2.5 of \citet{neyman_binary_2021} additionally shows that $f'(\hat p) > 0$ almost everywhere on $(0, 1)$. If $f$ were not strictly increasing, then there exists $x < y$ such that $f(x) = f(y)$. Since $f$ is weakly increasing given that $f'(\hat p) \geq 0$ everywhere, as in the proof of Lemma 2.5, it must be that $f$ is constant on $(x, y)$, so that $f'(\hat p) = 0$ on $(x ,y)$. However, this result would contradict the fact that $f'(\hat p) > 0$ almost everywhere because $(x, y)$ has non-zero measure. Hence, it must be that $f$ is strictly increasing. 
\end{proof}




\begin{restatable}{lemma}{driftderivatives}\label{lemma:drift-derivatives}
    Let $f$ be a proper scoring rule. For $\hat p \in (0, 1)$, it holds that
    \begin{align*}
        \partial_{\hat p} r_\driftmodel(\hat p) &= \alpha (f(\hat p) - f(1 - \hat p)) + f'(\hat p) \left(\frac{\alpha \hat p + (1 - \alpha) p  - \hat p}{1 - \hat p}\right) -q (\hat p - p) c.\\
    \end{align*}
\end{restatable}
\begin{proof}
We make extensive use of the fact that for proper scoring rules, $\hat p f'(\hat p) = (1 - \hat p) f'(1 - \hat p)$ for $\hat p \in (0, 1)$. 
\begin{align*}
    \partial_{\hat p} r(\hat p) &= \phi_1'(\hat p)(f(\hat p) - f(1 - \hat p)) + \phi_1(\hat p) f'(\hat p) - (1 - \phi_1(\hat p))f'(1 - \hat p) -q (\hat p - p) c\\
        &=\phi_1'(\hat p)(f(\hat p) - f(1 - \hat p)) + f'(\hat p) \left(\phi_1(\hat p)- (1 - \phi_1(\hat p))\frac{\hat p}{1 - \hat p}\right) -q (\hat p - p) c\\
        &= \phi_1'(\hat p)(f(\hat p) - f(1 - \hat p)) + f'(\hat p) \left(\frac{\phi_1(\hat p) (1 - \hat p)- \hat p (1 - \phi_1(\hat p))}{1 - \hat p}\right) -q (\hat p - p) c\\
        &= \phi_1'(\hat p) (f(\hat p) - f(1 - \hat p)) + f'(\hat p) \left(\frac{\phi_1(\hat p)  - \hat p}{1 - \hat p}\right) -q (\hat p - p) c\\
        &= \alpha (f(\hat p) - f(1 - \hat p)) + f'(\hat p) \left(\frac{\alpha \hat p + (1 - \alpha) p  - \hat p}{1 - \hat p}\right) -q (\hat p - p) c.
\end{align*}
\end{proof}



\end{document}